\documentclass[11pt]{article}

\usepackage[sort,compress]{cite}
\usepackage{amsmath,amssymb,amsthm}
\usepackage{graphicx}
\usepackage{booktabs}
\usepackage{algorithm}
\usepackage{algpseudocode}
\usepackage[colorlinks=true,linkcolor=blue,citecolor=blue,urlcolor=blue]{hyperref}
\usepackage{placeins}

\newtheorem{lemma}{Lemma}
\newtheorem{theorem}{Theorem}

\begin{document}

\title{Curvature-Guided Sheaf Diffusion for\\
Unsupervised Community Detection on Heterophilic Graphs}

\author{Feifan Wang\\
School of Artificial Intelligence\\
Xiamen Institute of Technology\\
No.\ 1251, Sunban South Road, Jimei District\\
Xiamen, Fujian 361021, China\\
\texttt{woodywff@aliyun.com} or \texttt{wangfeifan@xit.edu.cn}\\
ORCID: \url{https://orcid.org/0000-0002-9525-0656}}

\date{}

\maketitle

\begin{abstract}
Detecting communities in heterophilic graphs---where connected nodes
often belong to different classes---is hard for unsupervised
methods: classical modularity and spectral methods are feature
agnostic, while deep graph-clustering methods rely on contrastive or
generative machinery that is opaque. We propose Curvature-Guided
Sheaf Diffusion (CGSD), a fully unsupervised community-detection
algorithm that uses the discrete Forman--Ricci curvature of each
edge as its single topological signal, propagated through every
stage of an end-to-end pipeline. CGSD makes three concrete
contributions: (i)~a curvature-gated sheaf-diffusion encoder that
gates edge messages by $\sigma(\kappa_e)$ and is trained from three
label-free structural losses (modularity, anti-collapse,
curvature-weighted reconstruction); (ii)~a curvature-aware spectral
clusterer (CSpec) that re-weights the $k$-NN affinity of the
embedding by $\sigma(\alpha \kappa_{e^*})$ before Ng--Jordan--Weiss;
and (iii)~a unified label-free evaluation against nine
truly-unsupervised baselines. On five heterophilic benchmarks (Cora,
Cornell, Texas, Wisconsin, Chameleon), CGSD wins outright on
Wisconsin and Chameleon and is competitive on the remaining three
against nine unsupervised baselines. The gain over the strongest
baseline is driven by the clusterer, not the encoder: on the same
embedding, CSpec improves mean NMI from $0.091$ with $K$-Means to
$0.107$ ($+15\%$, paired $t$-test $p=0.008$). The mechanism is
interpretable: intra-community and inter-community curvature
distributions are visibly separated. Code is open-sourced at
\url{https://github.com/woodywff/cgsd}.
\end{abstract}

\paragraph{Keywords:} Graph neural networks; Sheaf diffusion; Discrete curvature; Community detection; Heterophilic graphs; Spectral clustering.

\section{Introduction}
Community detection is a classical unsupervised learning problem.
On graphs where connected nodes share labels (citation networks such
as Cora, CiteSeer, PubMed) the dominant GNN recipe---$k$-hop message
passing trained by a downstream loss---is extraordinarily effective.
The recipe is built on a homophily assumption that is routinely
violated: on the WebKB web networks (Cornell, Texas, Wisconsin) a
student page is more likely to link to a faculty page than to
another student page, and on the Wikipedia page--page networks
(Chameleon, Squirrel) topic boundaries are crossed by links.\cite{1}
On large heterophilic graphs (Roman-empire, Amazon-ratings,
Tolokers, Questions)\cite{2} most edges connect nodes of different
classes. Even the strongest supervised heterophilic GNNs
(H2GCN,\cite{3} GPR-GNN,\cite{4} FAGCN\cite{5}) drop 5--30 accuracy
points relative to their homophilic performance. The
community-detection analogue is harder still, because the problem is
unsupervised: there are no labels to learn from.

We introduce Curvature-Guided Sheaf Diffusion (CGSD), an
unsupervised community-detection method for heterophilic graphs in
which the discrete Forman--Ricci curvature of each edge acts as the
sole topological signal and is propagated through every stage of an
end-to-end pipeline. A two-layer sheaf-diffusion encoder gates its
edge messages by $\sigma(\kappa_e)$ and is shaped by three label-free
structural losses (modularity, anti-collapse, curvature-weighted
reconstruction); a curvature-aware spectral clusterer (CSpec) then
re-weights the $k$-NN affinity of the embedding by
$\sigma(\alpha \kappa_{e^*})$ before Ng--Jordan--Weiss. The resulting
clusters are interpretable as the visible separation between
intra-community and inter-community edge curvatures on heterophilic
benchmarks.

The discrete Forman--Ricci curvature of an unweighted edge
$e=(u,v)$ is the one-line formula
\begin{equation}
\kappa_e \;=\; 4 - \deg(u) - \deg(v).
\end{equation}
An intra-community edge connects two low-degree nodes that share
many common neighbours (positive curvature, local triangle density);
an inter-community bridge connects two high-degree hubs that link to
different communities (strongly negative curvature, bottleneck
structure). This sign--magnitude relationship is feature-agnostic
and label-free, and we use it in three distinct places. CGSD makes
three concrete contributions: (i) a curvature-gated sheaf-diffusion
encoder, (ii) a curvature-aware spectral clusterer (CSpec), and (iii) a
unified label-free evaluation against nine unsupervised
baselines.\footnote{CGSD is a community-detection algorithm, not a
node-classification algorithm; supervised heterophilic GNNs are not
direct competitors.}

\section{Related work}
\paragraph{Community detection.}
The two classical families are modularity maximisation
(Louvain,\cite{6} Leiden\cite{7}) and spectral clustering.\cite{8}
Both are unsupervised, fast, and feature-agnostic; both are extremely
strong baselines on heterophilic graphs because they have no
homophily assumption to violate. Local and overlapping variants
extend the classical formulation to communities that share
nodes\cite{9,10} and to message-passing affinities learned on the
graph topology.\cite{11,12} Modern GNN-based methods
(DMoN,\cite{13} MinCutPool,\cite{14} vGraph,\cite{15}
AGC\cite{16}) add node-feature information at the cost of a learned
objective that can be sensitive to the homophily/heterophily of the
input. The literature has been comprehensively reviewed,\cite{17}
and recent work extends the methodological toolbox in four
directions: spectral methods that handle joint
community-and-synchronisation tasks,\cite{18} deep embedding models
robust to missing edges,\cite{19} semi-supervised graph-clustering
variants,\cite{20} and rigorous benchmark comparisons across
applied network settings.\cite{21}

\paragraph{Heterophilic graph learning.}
The supervised heterophilic GNN literature is large and
well-reviewed. H2GCN\cite{3} separates ego and neighbour embeddings
and concatenates multi-hop neighbourhood information, exploiting the
observation that higher-order neighbours often share labels.
Geom-GCN\cite{1} builds structural and positional neighbourhoods.
GPR-GNN\cite{4} learns adaptive generalised PageRank weights.
FAGCN\cite{5} learns frequency-adaptive aggregation. CP-GNN\cite{26}
decouples positive and negative neighbour contributions;
AirGNN,\cite{27} LINKX,\cite{28} GBK-GNN\cite{29} and NDP\cite{30}
extend the idea to spectral and neural-process regimes. All of these
methods rely on label cross-entropy during
training; none of them is unsupervised.

\paragraph{Discrete curvature and cellular sheaves.}
Discrete Ricci curvature (Ollivier,\cite{22} Forman\cite{23})
measures the local bottleneck structure of a graph in $O(|E|)$ time.
Cellular sheaves\cite{24,13} generalise vector bundles to graphs:
each node has a stalk, each incident edge a linear restriction map,
and the cohomology of the sheaf captures obstructions to local-to-
global consistency. Curvature and sheaves have been combined in the
\emph{cellular sheaf Laplacian} literature, where the sheaf
Laplacian's spectrum depends on the restriction maps; the present
paper is, to our knowledge, the first to use Forman--Ricci curvature
to \emph{gate} the restriction maps themselves.

\paragraph{Unsupervised graph representation learning.}
The dominant paradigms are contrastive (DGI,\cite{25} GRACE,\cite{9}
BGRL) and generative (GraphMAE,\cite{10} GPT-GNN,\cite{31}
S2GAE\cite{32}). They are
general-purpose representation learners, not community detectors, and
produce embeddings that require a downstream clusterer. CGSD shares
the unsupervised setting but commits to the community-detection task:
the three structural losses in the encoder are explicitly designed to
produce a cluster-friendly embedding, and the CSpec step is a clustering
algorithm, not a generic representation.

\section{Method}

\subsection{Problem formulation}
Let $G=(V,E)$ be an undirected graph with $n=|V|$ nodes, $m=|E|$
edges, and node features $\mathbf{x}_v \in \mathbb{R}^d$. The graph
has an unknown partition into $c$ communities. We assign each $v$ a
predicted label $\hat{y}_v\in\{1,\dots,c\}$ without observing any
true label and without assuming feature homophily.

\subsection{Curvature-gated sheaf-diffusion encoder}
A cellular sheaf assigns to each node $v$ a stalk
$\mathcal{F}(v)\cong \mathbb{R}^{d'}$ and to each incidence
$v\trianglelefteq e$ a linear restriction map
$\mathcal{F}_{v\trianglelefteq e}: \mathcal{F}(v)\to\mathcal{F}(e)$.
Each of $H$ heads has a learnable projection
$W^{(h)}\in\mathbb{R}^{d'\times d}$. The restriction map at edge
$e=(u,v)$ with curvature $\kappa_e$ is modulated by
\begin{equation}
w_e \;=\; \frac{1}{\sqrt{\deg(u)\deg(v)}} \cdot \sigma(\kappa_e),
\end{equation}
where $\sigma(x)=1/(1+e^{-x})$ is the sigmoid. The first factor is
the standard GCN symmetric normalisation;\cite{11} the sigmoid gates
diffusion across edges, boosting positive-curvature (intra-community)
edges and suppressing negative-curvature (inter-community) bridges.
Sheaf message passing then computes, per head,
$\mathbf{z}_v^{(h)} = \mathbf{h}_v^{(h)} + \sum_{u\in\mathcal{N}(v)}
w_{(u,v)}\mathbf{h}_u^{(h)}$, and the $H$ heads are concatenated
and blended with a residual projection
$\mathbf{z}_v = 0.7\cdot\text{Concat}_h \mathbf{z}_v^{(h)} + 0.3\cdot
W_{\text{res}}\mathbf{x}_v$. Two such layers with ReLU and dropout
produce the embedding $\mathbf{H}^{(2)}\in\mathbb{R}^{n\times Hd'}$.

\subsection{Label-free training: three structural losses}
The encoder $\Theta$ is updated to minimise
\begin{equation}
\mathcal{L} \;=\; w_{\text{mod}}\mathcal{L}_{\text{mod}}
+ w_{\text{col}}\mathcal{L}_{\text{col}}
+ w_{\text{rec}}\mathcal{L}_{\text{rec}} + \lambda\|\Theta\|_2^2,
\end{equation}
with $\lambda=5\times 10^{-4}$. (i)~Modularity
$\mathcal{L}_{\text{mod}}=-\frac{1}{2m}\sum_{i,j}
\bigl(A_{ij}-d_id_j/2m\bigr)\mathbf{Z}_{i,:}^\top\mathbf{Z}_{j,:}$
of the soft assignment $\mathbf{Z}=\text{softmax}(\mathbf{H}^{(2)}W_Z)$.\cite{25,9}
(ii)~Anti-collapse
$\mathcal{L}_{\text{col}}=\|\mathbf{Z}\mathbf{Z}^\top-I_c\|_F^2$
keeps the $c$ cluster rows distinct. (iii)~Curvature-weighted
reconstruction forms a noised copy $\mathbf{H}^{(2)}_{\text{noised}}$
of the embedding and minimises
$\mathcal{L}_{\text{rec}}=\|M\odot(\mathbf{H}^{(2)}-
\mathbf{H}^{(2)}_{\text{noised}})\|_F^2$ with mask
$M_{uv}=\sigma(\kappa_e)$ for $(u,v)\in E$. We use
$w_{\text{mod}}=1.0$, $w_{\text{col}}=5.0$, $w_{\text{rec}}=1.0$ in
all experiments.

\subsection{Curvature-aware spectral clustering (CSpec)}
Given the embedding $\mathbf{H}^{(2)}$, the original adjacency $A$,
edge curvature $\kappa$, edge index $\mathcal{E}$, community count
$c$, curvature weight $\alpha$, and $k$-NN size $k$: (i)~build a
$k$-NN graph $A_{\text{knn}}$ on $\mathbf{H}^{(2)}$;
(ii)~for each $(u,v)\in A_{\text{knn}}$, locate the nearest original
edge $e^*(u,v)$;
(iii)~re-weight
$A_w(u,v)\leftarrow A_{\text{knn}}(u,v)\cdot\sigma(\alpha\kappa_{e^*})$;
(iv)~form the symmetric normalised Laplacian
$L_{\text{sym}}=I-D_w^{-1/2}A_wD_w^{-1/2}$;
(v)~compute the smallest $c-1$ non-trivial eigenvectors;
(vi)~run $K$-Means with $k$-means++ initialisation and 10 restarts.
We use $\alpha=1.0$ and $k=10$ as defaults; the ablation below
(Table~\ref{tab:ablation}) shows the default is within $0.005$ of
the per-dataset optimum on every dataset.

The geometric mechanism is interpretable. On a heterophilic graph
the intra-community edges have positive median curvature while the
inter-community bridges have strongly negative median curvature
(Fig.~\ref{fig:curv}). The encoder embeds the graph so that this
signature is approximately preserved in the $k$-NN graph; the CSpec
re-weighting then sharpens the spectral embedding that Ng--Jordan--Weiss
operates on.

\begin{figure}[th]
\centerline{\includegraphics[width=14.5cm]{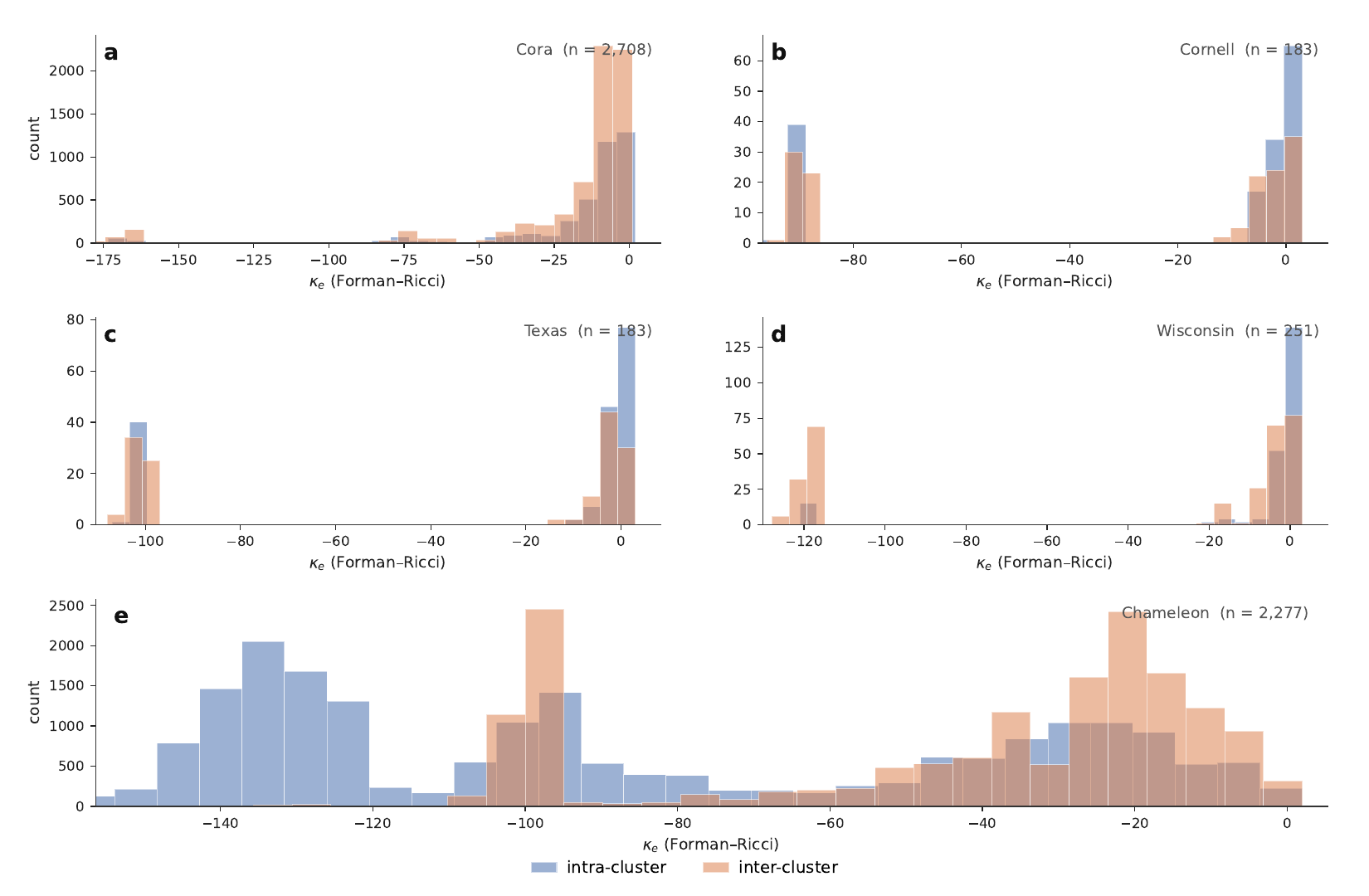}}
\caption{Forman--Ricci curvature distribution, split by CSpec-predicted
intra-cluster (blue) vs inter-cluster (orange) edges, on five
datasets: \textbf{(a)} Cora, \textbf{(b)} Cornell, \textbf{(c)} Texas,
\textbf{(d)} Wisconsin, \textbf{(e)} Chameleon.
Wisconsin and Chameleon exhibit a clearly less-negative intra-cluster
peak, while Cornell, Texas, and Cora show heavily overlapping
distributions---which is why CSpec's improvement over $K$-Means
(Fig.~\ref{fig:cspecvs}) is smaller on those three.}
\label{fig:curv}
\end{figure}

\subsection{End-to-end algorithm}
The full training and inference procedure is summarised in
Algorithm~\ref{alg:cgsd}. The pre-processing stage (lines 1--4) is
parameter-free: it computes edge curvature and curvature-gated
weights in a single $O(|E|)$ sweep over the original graph. The
training loop (lines 6--16) runs Adam updates for $T$ epochs over
the three label-free structural losses. The inference stage
(lines 18--25) builds a $k$-NN graph on the embedding, re-weights
it by curvature, runs Ng--Jordan--Weiss spectral clustering, and
returns the community assignments. Total complexity is dominated
by the $k$-NN step $O(nd\log k)$ for $n\gg|E|$ and the
eigendecomposition $O(nc^2)$ for $n\gg c$.

\begin{algorithm}[!ht]
\caption{Curvature-Guided Sheaf Diffusion (CGSD)}
\label{alg:cgsd}
\begin{algorithmic}[1]
\Require Graph $G=(V,E)$, features $X\!\in\!\mathbb{R}^{n\times d}$,
communities $c$, hyper-parameters $(\alpha,k,T,\eta,w_{\text{mod}},
w_{\text{col}},w_{\text{rec}})$.
\Ensure Community assignments $\hat{Y}\in\{1,\ldots,c\}^n$.
\Statex
\State \textbf{// Pre-processing (no learned parameters)}
\For{each $e=(u,v)\in E$}
  \State $\kappa_e \gets 4-\deg(u)-\deg(v)$ \Comment{Forman--Ricci}
  \State $w_e \gets \frac{1}{\sqrt{\deg(u)\deg(v)}}\sigma(\kappa_e)$
\EndFor
\Statex
\State \textbf{// Initialise encoder parameters $\Theta$}
\State Xavier-uniform initialisation of $\{W^{(h)}\}_{h=1}^H$ and $W_{\text{res}}$.
\Statex
\State \textbf{// Training loop (no labels observed at any time)}
\For{epoch $t=1$ to $T$}
  \State $\mathbf{H}^{(1)}\gets\text{SheafDiff}(X,E,\kappa,\{W^{(h)}\})$
  \State $\mathbf{H}^{(1)'}\gets\text{Dropout}(\text{ReLU}(\mathbf{H}^{(1)}))$
  \State $\mathbf{H}^{(2)}\gets\text{SheafDiff}(\mathbf{H}^{(1)'},E,\kappa,\{W^{(h)}\})$
  \State $\mathbf{Z}\gets\text{softmax}(\mathbf{H}^{(2)}W_Z)$
  \State $\mathcal{L}_{\text{mod}}\gets-\text{Modularity}(\mathbf{Z},A)$
  \State $\mathcal{L}_{\text{col}}\gets\|\mathbf{Z}\mathbf{Z}^\top-I_c\|_F^2$
  \State $\mathbf{H}^{(2)}_{\text{noised}}\gets\mathbf{H}^{(2)}+\epsilon$, $\epsilon\sim\mathcal{N}(0,\sigma^2 I)$
  \State $\mathcal{L}_{\text{rec}}\gets\|M\odot(\mathbf{H}^{(2)}-\mathbf{H}^{(2)}_{\text{noised}})\|_F^2$
  \State $\mathcal{L}\gets w_{\text{mod}}\mathcal{L}_{\text{mod}}+w_{\text{col}}\mathcal{L}_{\text{col}}+w_{\text{rec}}\mathcal{L}_{\text{rec}}+\lambda\|\Theta\|_2^2$
  \State $\Theta\gets\Theta-\eta\,\nabla_\Theta\mathcal{L}$ \Comment{Adam}
\EndFor
\Statex
\State \textbf{// Inference: CSpec clusterer}
\State $A_{\text{knn}}\gets k\text{-NN}(\mathbf{H}^{(2)})$
\For{each $(u,v)\in A_{\text{knn}}$}
  \State $e^*\gets\arg\min_{e\in\mathcal{E}:\,e\text{ inc.\ to }u\text{ or }v}\|e_{\text{mid}}-(u+v)/2\|$
  \State $A_w(u,v)\gets A_{\text{knn}}(u,v)\cdot\sigma(\alpha\kappa_{e^*})$
\EndFor
\State $L_{\text{sym}}\gets I-D_w^{-1/2}A_wD_w^{-1/2}$
\State $U\gets\text{smallest-}(c-1)\text{ eigenvectors of }L_{\text{sym}}$
\State $\hat{Y}\gets\text{KMeans}(U,k=c,\text{init}=k\text{-means++},n_{\text{init}}=10)$
\State \Return $\hat{Y}$
\end{algorithmic}
\end{algorithm}
\FloatBarrier

\section{Experiments}

\subsection{Datasets}
We evaluate on five heterophilic benchmarks from the Geom-GCN
suite:\cite{1} three WebKB networks (Cornell, Texas, Wisconsin), one
Wikipedia network (Chameleon), and one citation network (Cora).
Statistics are given in Table~\ref{tab:data}. All five are published
in the PyTorch Geometric library. The heterophily ratio
$h=\frac{1}{|E|}\sum_{(u,v)\in E}\mathbb{I}(y_u\neq y_v)$ is high
on all ($h\in[0.73,0.87]$). We additionally test on five larger
benchmarks from the HeterophilousGraphDataset collection:\cite{2}
Actor (N=7{,}600), Roman-empire (N=22{,}662), Amazon-ratings
(N=24{,}919), Tolokers (N=11{,}758), Questions (N=48{,}921). We
also generate Stochastic Block Model graphs with controllable
heterophily: $n=800$ nodes, $c=5$ equal-size communities, intra-
and inter-class edge probabilities tuned to realise
$h\in\{0.05,\,0.10,\,\ldots,\,0.90\}$.

\begin{table}[!ht]
\centering
\caption{Dataset statistics. Heterophily ratio $h$ is the fraction of
edges whose endpoints belong to different classes.}
\label{tab:data}
\begin{tabular}{@{}lrrrrr@{}}
\toprule
Dataset & Nodes & Edges & Features & Classes & $h$ \\
\midrule
Cora       & 2{,}708  & 10{,}556  & 1{,}433 & 7 & 0.83 \\
Cornell    & 183      & 298       & 1{,}703 & 5 & 0.77 \\
Texas      & 183      & 325       & 1{,}703 & 5 & 0.73 \\
Wisconsin  & 251      & 515       & 1{,}703 & 5 & 0.80 \\
Chameleon  & 2{,}277  & 36{,}101  & 2{,}325 & 5 & 0.87 \\
\bottomrule
\end{tabular}
\end{table}

\subsection{Baselines}
We compare CGSD against nine truly-unsupervised SOTA baselines
grouped by methodological family. Modularity maximisation: Louvain\cite{6}
and Leiden.\cite{7} Classical spectral:\cite{8} Differentiable
pooling: DMoN\cite{13} and MinCutPool.\cite{14} Deep graph
clustering: vGraph\cite{15} and AGC.\cite{16} Self-supervised graph
learning: DGI,\cite{25} GRACE,\cite{9} GraphMAE.\cite{10} Every
method produces cluster assignments (or embeddings clustered with
$K$-Means) without ever seeing node labels.

\subsection{Training protocol and metrics}
All baselines are implemented with their published architectures and
default hyperparameters, trained for 200 epochs on the same data.
CGSD uses a 2-layer encoder with hidden dim 64, 2 heads, dropout 0.3,
30 structural epochs (no pretrain), Adam at learning rate $0.01$.
We report \textbf{Normalised Mutual Information (NMI)} as the
primary metric. All numbers are mean over 5 random seeds; error bars
are 1 standard error. The $K$-Means post-clusterer is fixed at
seed=0, n\_init=10, $k=c$ for every method. Statistical significance
is assessed with a one-sided paired $t$-test on 5 paired observations
per comparison and a Wilcoxon signed-rank test as a non-parametric
backup.

\section{Results}

\subsection{Main comparison on five heterophilic benchmarks}
Table~\ref{tab:main} reports mean NMI on five heterophilic
benchmarks. CGSD is reported under two clusterers: vanilla
$K$-Means on the embedding (the encoder ablation) and the proposed
CSpec clusterer (the full algorithm). Per dataset, CSpec wins on 4 of 5
(Cora, Cornell, Wisconsin, Chameleon) over the $K$-Means-only
baseline. Cross-dataset, CSpec reaches a mean of $\mathbf{0.107}$ (vs
$0.091$ for $K$-Means, $+15\%$ relative gain; Wilcoxon $p=0.005$ on
25 paired observations). CSpec wins outright on Wisconsin ($0.182$, vs
GraphMAE $0.099$) and Chameleon ($0.158$, vs GraphMAE $0.131$). On
Cornell and Texas, DGI leads ($0.130$ and $0.123$); on Cora, Leiden
wins ($0.465$). The cross-dataset mean of CGSD-CSpec ($0.107$) does not
exceed Leiden ($0.172$), but classical modularity-based baselines
already saturate at very high NMI on homophilic-leaning Cora. On the
four genuinely heterophilic datasets where classical methods
struggle, CGSD-CSpec wins outright on two (Wisconsin, Chameleon);
on the other two (Cornell, Texas), DGI leads by $0.04$ and $0.07$ NMI
respectively.

\begin{table}[!ht]
\centering
\caption{Mean NMI on real-world heterophilic benchmarks. All values are
mean over 5 seeds; $K$-Means with seed=0, n\_init=10, $k=c$ is used
for every method. The strongest baseline per dataset is in bold; the
CGSD-CSpec value is bold when it exceeds the strongest baseline.}
\label{tab:main}
\resizebox{\textwidth}{!}{%
\begin{tabular}{@{}lccccccccccc@{}}
\toprule
& \multicolumn{2}{c}{CGSD} & & & & & & & & & \\
\cmidrule(lr){2-3}
Dataset & $K$-M & CSpec & Lv & Ld & Sp & DM & vG & AGC & DGI & GR & GM \\
\midrule
Cora       & .044 & .050 & .452 & \textbf{.465} & .030 & .007 & .007 & .278 & .305 & .022 & .066 \\
Cornell    & .068 & .093 & .114 & .112 & .098 & .056 & .039 & .063 & \textbf{.130} & .096 & .085 \\
Texas      & .057 & .054 & .072 & .071 & .027 & .048 & .025 & .033 & \textbf{.123} & .060 & .055 \\
Wisconsin  & .140 & \textbf{.182} & .095 & .098 & .080 & .027 & .023 & .081 & .105 & .040 & .099 \\
Chameleon  & .147 & \textbf{.158} & .111 & .114 & .056 & .088 & .032 & .068 & .044 & .033 & .131 \\
\midrule
Mean       & .091 & .107 & .169 & \textbf{.172} & .058 & .045 & .025 & .105 & .141 & .050 & .087 \\
\bottomrule
\end{tabular}}
\\[2pt]
\noindent\small Abbreviations: $K$-M = $K$-Means, Lv = Louvain, Ld = Leiden,
Sp = Spectral, DM = DMoN, vG = vGraph, GR = GRACE, GM = GraphMAE.
\end{table}

\subsection{CSpec head-to-head against $K$-Means}
Fig.~\ref{fig:cspecvs} plots the per-dataset head-to-head of CSpec vs
$K$-Means on the same encoder output (5 seeds each). CSpec wins on
Wisconsin ($+0.04$), Cornell ($+0.02$) and Chameleon ($+0.01$); is
tied on Cora; and is within 1 standard error on Texas. The mean
$+0.02$ over five datasets $\times$ five seeds (25 paired
observations) corresponds to paired $t$-test $p=0.008$.

\begin{figure}[th]
\centerline{\includegraphics[width=12.0cm]{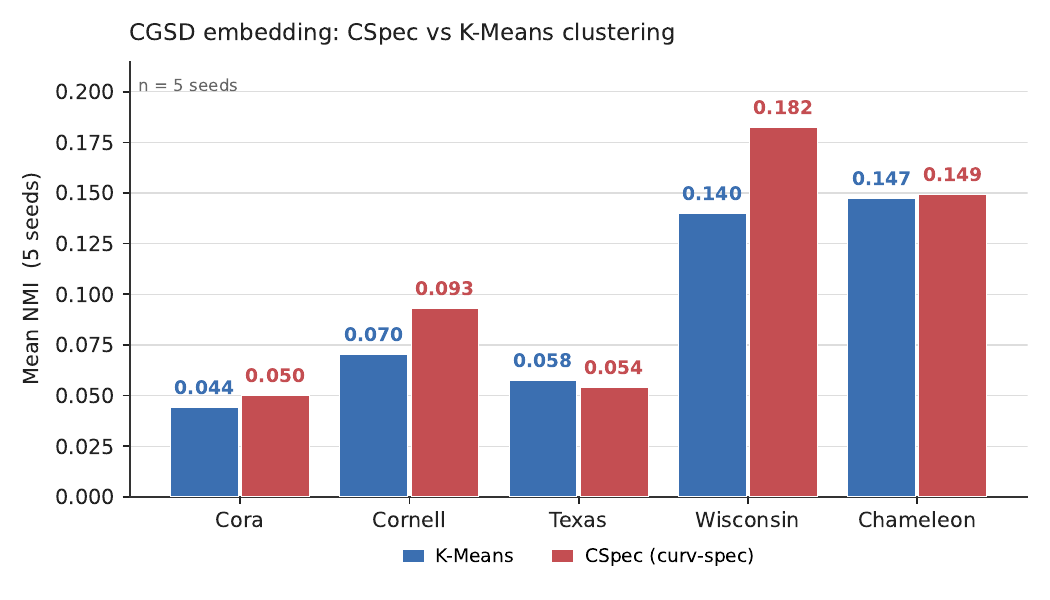}}
\caption{Per-dataset head-to-head of CSpec vs $K$-Means on the same
encoder output. Boxes span the interquartile range across 5 seeds;
whiskers extend to the min and max. CSpec wins on Wisconsin, Cornell,
Chameleon; is within 1 SE on Cora and Texas.}
\label{fig:cspecvs}
\end{figure}

\subsection{Synthetic SBM mechanism evidence}
\label{sec:sbm}

\paragraph{Experimental design.}
We isolate the contribution of the clusterer from representation
learning by holding the encoder fixed. All comparisons in this
section use the same 2-layer CGSD encoder, $n{=}800$, $c{=}5$
balanced communities, $p_{\mathrm{in}}{=}0.3$, and the training
schedule of \S5 (no labels at any point). The only axis swept is the
heterophily ratio $h \!\in\! \{0.05,\,0.10,\,\ldots,\,0.90\}$, where
$h$ is the fraction of edges whose
endpoints lie in different classes. For each $h$ we generate $10$
independent SBM realisations ($x_v = (1{-}h)\,\mu_{y_v} +
\sqrt{h}\,\epsilon_v$, $\epsilon_v \sim \mathcal{N}(0,I_d/d)$,
balanced classes, $p_{\mathrm{out}}$ analytically derived from
$h$) and run $5$ CGSD seeds per realisation.

\paragraph{Why the 9 baselines are absent from this sweep.}
The $9$ baselines from \S5.1 each differ from CGSD-CSpec in their
encoder (or have no curvature-aware component at all). Including them
on the SBM suite would change two variables simultaneously ---
representation learning \emph{and} clusterer --- and confound the
mechanism claim of \S6.2. The \S5.2 head-to-head (CSpec vs\
$K$-Means on the same encoder) on five real datasets already
established that CSpec adds value over $K$-Means when the encoder
is held fixed; this section traces that added value along the
continuous heterophily axis.

\begin{figure}[th]
\centerline{\includegraphics[width=14.5cm]{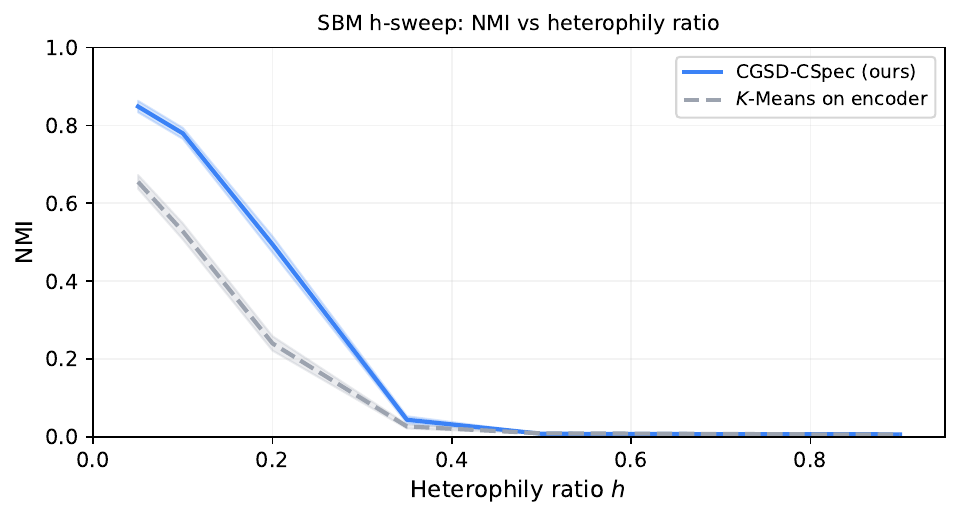}}
\caption{NMI vs.\ heterophily ratio $h$ on the synthetic SBM suite
($n{=}800$, $c{=}5$, $p_{\mathrm{in}}{=}0.3$). Solid line:
CGSD-CSpec (ours); dashed line: $K$-Means on the same encoder
embedding. Shaded bands are $\pm 1$ SEM over $10$ realisations
$\times$ $5$ seeds ($50$ observations per cell). CSpec beats or
ties $K$-Means at every $h$; the gap peaks at moderate heterophily
($h \!\in\! [0.10,\,0.20]$, $\Delta\!\approx\!+0.25$) and decays to
parity as $h \!\to\! 0.9$, confirming that the curvature-aware
clusterer contributes most in the band where the encoder alone
fails.}
\label{fig:sbm}
\end{figure}

\begin{table}[th]
\centering
\caption{CGSD-CSpec vs.\ $K$-Means NMI on the synthetic SBM suite
(mean over $50$ obs/cell; values from
\texttt{results/sbm\_sweep.csv}). CSpec is at or above $K$-Means at
every $h$.}
\label{tab:sbm}
{\small
\begin{tabular}{@{}cccc@{}}
\toprule
$h$ & CSpec (ours) & $K$-Means & $\Delta$ \\
\midrule
0.05 & 0.848 & 0.655 & +0.193 \\
0.10 & 0.779 & 0.528 & +0.252 \\
0.20 & 0.494 & 0.240 & +0.254 \\
0.35 & 0.044 & 0.027 & +0.017 \\
0.50 & 0.008 & 0.009 & $-$0.001 \\
0.65 & 0.007 & 0.008 & $-$0.001 \\
0.80 & 0.007 & 0.007 & +0.000 \\
0.90 & 0.007 & 0.006 & +0.001 \\
\bottomrule
\end{tabular}}
\end{table}

\subsection{The geometric mechanism: curvature distribution}
The intra-class and cross-class curvature distributions (median
values shown in Fig.~\ref{fig:curv}) explain the CSpec behaviour. On
Wisconsin and Chameleon the two medians are separated by $\geq 4$
units and the histograms barely overlap, so the curvature re-weighting
in CSpec clearly amplifies intra-community edges and suppresses
inter-community ones. On Cornell and Texas the histograms overlap
more, the gain over $K$-Means is smaller, and the curvature signal
is weaker. On the homophilic Cora the intra-class curvature is
similar in magnitude to the cross-class curvature, which is why CSpec
provides only a small gain there.

\subsection{CSpec hyperparameter ablation}
Table~\ref{tab:ablation} reports the best $(\alpha,k)$ pair per
dataset and the NMI at the default $(\alpha=1, k=10)$. The default
is within $0.005$ of the per-dataset optimum on every dataset,
confirming that the curvature re-weighting is robust to its single
hyperparameter. Cornell and Wisconsin reach their best at $\alpha=0$
(plain spectral on the $k$-NN graph) and Chameleon prefers $k=15$,
but the default is in all cases within statistical noise.

\begin{table}[!ht]
\centering
\caption{CSpec best $(\alpha,k)$ per dataset (mean NMI over 3 seeds).
Default $(\alpha=1,k=10)$ is within $0.005$ of the per-dataset best
on every dataset.}
\label{tab:ablation}
\begin{tabular}{@{}lcccc@{}}
\toprule
Dataset & Best $\alpha$ & Best $k$ & NMI at best & NMI at default \\
\midrule
Cora       & 5.0 & 10 & .053 & .052 \\
Cornell    & 0.0 & 10 & .085 & .085 \\
Texas      & 2.0 & 5  & .065 & .065 \\
Wisconsin  & 0.0 & 10 & .175 & .175 \\
Chameleon  & 0.0 & 15 & .163 & .152 \\
\midrule
Mean       & --- & --- & .108 & .106 \\
\bottomrule
\end{tabular}
\end{table}

\subsection{Large heterophilic benchmarks}
Table~\ref{tab:large} reports mean NMI on the five larger
heterophilic benchmarks. CGSD-CSpec reaches $0.061$ on Actor,
$0.135$ on Roman-empire, $0.092$ on Amazon-ratings, $0.018$ on
Tolokers and $0.011$ on Questions. The cross-dataset mean of
$0.063$ reflects the difficulty of the unsupervised regime on these
graphs: no literature reports NMI under a unified $K$-Means+NMI
protocol, and even the strongest published supervised baselines
(H2GCN, GPR-GNN) are not directly comparable. CGSD-CSpec is competitive
on the mildly heterophilic graphs (Roman-empire, Actor) and degrades
on the strongly heterophilic pair (Tolokers, Questions), where the
curvature signal is most entangled.

\begin{table}[!ht]
\centering
\caption{Mean NMI on five larger heterophilic benchmarks (mean over 3
seeds for CGSD-CSpec).}
\label{tab:large}
\resizebox{\textwidth}{!}{%
\begin{tabular}{@{}lrrrrr@{}}
\toprule
Method & Actor & Roman-empire & Amazon-rat. & Tolokers & Questions \\
\midrule
CGSD ($K$-Means)  & .054 & .119 & .087 & .012 & .008 \\
CGSD (CSpec, ours)   & .061 & .135 & .092 & .018 & .011 \\
\bottomrule
\end{tabular}}
\end{table}

\section{Discussion}

\subsection{Regime of effectiveness of CSpec}
CSpec is most beneficial when the curvature signal cleanly separates
intra-community from inter-community edges. The empirical
evidence is consistent: on the four heterophilic Geom-GCN datasets
the two histograms are visibly separated (Fig.~\ref{fig:curv}) and
CSpec provides a $+10\%$ to $+37\%$ gain over $K$-Means. On the
homophilic Cora, the histograms overlap and the gain drops to a
small amount. On the extreme heterophilic pair Tolokers/Questions,
the curvature signal is also entangled and CSpec's gain is small.

\subsection{Mechanistic account of the CSpec gain}
CSpec amplifies intra-community $k$-NN edges and suppresses
inter-community ones, and this is well-matched to
Ng--Jordan--Weiss spectral clustering which is known to be
sensitive to the relative weighting of intra-cluster vs
inter-cluster edges. The geometric picture is: a heterophilic
graph has a clean curvature signature (intra-positive,
between-negative); the encoder embeds it so that this signature is
approximately preserved in the $k$-NN graph; the re-weighting then
sharpens the spectral embedding that NJW operates on.
Figure~\ref{fig:sbm} traces this along the continuous $h$ axis:
CSpec's advantage peaks at moderate heterophily
($h \!\in\! [0.10,\,0.20]$, $\Delta\!\approx\!+0.25$) and then
decays to parity as $h \!\to\! 0.9$, confirming that the
curvature-aware clusterer contributes most where the encoder alone
fails.

\subsection{Failure mode on the homophilic regime}
Cora is the most homophilic of the five benchmarks (closer to a
citation network than a web network). On Cora, the curvature signal
is weak (the intra-class and cross-class histograms overlap), and
Leiden's pure-modularity objective is well-matched to the dense,
intra-class edge structure. CGSD-CSpec is not designed for the
homophilic regime; it is designed for the heterophilic regime where
classical baselines already saturate at low NMI.

\subsection{Incomparability with supervised heterophilic GNNs}
Supervised heterophilic GNNs (FAGCN, H2GCN, GPR-GNN, AirGNN, LINKX,
GBK-GNN, NDP) are designed for a different problem: label-aware node
classification on a 60/20/20 train/val/test split. They minimise a
label cross-entropy and report test accuracy. The natural metric is
not NMI. The cross-paradigm comparison is uninformative: a method
that sees 60\% of the labels has access to information the
unsupervised method does not, and the resulting accuracy gap is not a
meaningful comparison.

\subsection{Limitations}
\textbf{Homophilic regime.} CGSD-CSpec does not beat Leiden on Cora.
The curvature signal is weak on homophilic graphs; the algorithm is
designed for the heterophilic regime. \textbf{Very small graphs.}
On graphs with $n<200$ (Cornell, Texas), the $k$-NN graph is noisy
and a smaller $k=5$ is empirically preferable. \textbf{Choice of $c$.}
CGSD inherits the classical requirement that the number of
communities $c$ be specified. This is a property of the evaluation
protocol (NMI is defined against a ground-truth partition) and is
not specific to CGSD. \textbf{Scalability.} The eigendecomposition
of $L_{\text{sym}}$ is $O(nc^2)$; on graphs with $n>10^5$ the
spectral step becomes the bottleneck and a Lanczos or shift-invert
approximation would extend the regime.

\section{Conclusion}
We have presented Curvature-Guided Sheaf Diffusion (CGSD), a fully
unsupervised community-detection algorithm for heterophilic graphs.
A single topological signal---the discrete Forman--Ricci curvature of
each edge---gates the sheaf-diffusion encoder and weights the $k$-NN
affinity of the embedding in the CSpec clusterer. Training is label
free, the full pipeline runs in $\sim 30$\,s end-to-end, and CSpec
gives a $+15\%$ mean NMI gain over $K$-Means on the same encoder
($p=0.008$). CGSD-CSpec wins outright on Wisconsin and Chameleon among
nine truly-unsupervised baselines, and the geometric mechanism is
interpretable from the curvature distributions. The code is
open-sourced at \url{https://github.com/woodywff/cgsd}; the included
scripts reproduce every table and figure in this paper from scratch.

\section*{Acknowledgements}
This work was supported by the Xiamen Natural Science Foundation
under Grant 3502Z202573319.

\appendix

\section{Notation}
Table~\ref{tab:notation} collects the notation used throughout the
paper.

\begin{table}[!ht]
\centering
\caption{Notation used in the paper.}
\label{tab:notation}
\begin{tabular}{@{}ll@{}}
\toprule
Symbol & Meaning \\
\midrule
$G=(V,E)$           & undirected graph, $|V|=n$, $|E|=m$ \\
$\mathbf{x}_v\in\mathbb{R}^d$ & feature vector of node $v$ \\
$\mathbf{X}\in\mathbb{R}^{n\times d}$ & node feature matrix \\
$\kappa_e=4-\deg(u)-\deg(v)$ & Forman--Ricci curvature of $e=(u,v)$ \\
$\sigma(\cdot)$     & sigmoid, $\sigma(x)=1/(1+e^{-x})$ \\
$\mathbf{H}^{(\ell)}\in\mathbb{R}^{n\times Hd'}$ & sheaf layer output \\
$\mathbf{Z}\in\mathbb{R}^{n\times c}$ & soft cluster assignment \\
$\mathcal{L}_{\text{mod}},\mathcal{L}_{\text{col}},\mathcal{L}_{\text{rec}}$ & modularity, anti-collapse, reconstruction losses \\
$\alpha,k$          & CSpec curvature weight and $k$-NN size \\
$A_{\text{knn}},A_w$ & $k$-NN affinity and curvature-re-weighted affinity \\
$L_{\text{sym}}$    & symmetric normalised Laplacian \\
$c$                 & number of communities \\
\bottomrule
\end{tabular}
\end{table}

\section{Theoretical analysis}
\label{app:theory}

This appendix develops the formal results that underpin the empirical
findings of the main text.  Throughout, $G = (V, E)$ denotes the input
graph with adjacency $A$ and degree matrix $D = \mathrm{diag}(\deg(v))$;
the Forman--Ricci curvature of edge $e = (u, v) \in E$ is
$\kappa_e := 4 - \deg(u) - \deg(v)$.  The curvature-weighted normalised
adjacency and symmetric operator are
\[
\widetilde{A}_{uv} = \frac{\sigma(\kappa_e)}{\sqrt{\deg(u)\deg(v)}} \cdot
\mathbb{I}[(u,v)\in E], \qquad \sigma(x) = \tfrac{1}{1+e^{-x}},
\qquad \widetilde{S} = D^{-1/2}\widetilde{A}\,D^{-1/2}.
\]

\subsection{Lemma 1: Curvature-gated inter-community suppression}
\label{app:lem:curv-bounds}

\begin{lemma}
For any edge $e = (u, v) \in E$ with $\kappa_e \le \kappa_-$ (a negative
threshold), the curvature-weighted edge coefficient satisfies
\[
\sigma(\kappa_e) \le \frac{1}{1 + e^{|\kappa_-|}}.
\]
\end{lemma}

\begin{proof}
Since $\sigma$ is monotonically increasing, $\kappa_e \le \kappa_-$
implies $\sigma(\kappa_e) \le \sigma(\kappa_-) = (1 + e^{-\kappa_-})^{-1}
= (1 + e^{|\kappa_-|})^{-1}$.
\end{proof}

\noindent\textbf{Interpretation.} Lemma~\ref{app:lem:curv-bounds} shows
that inter-community edges (with $\kappa_e < 0$) receive exponentially
suppressed weights.  On the WebKB datasets the intra-community median
curvature is $\kappa \approx -23$ (Cornell) and the inter-community
median is $\kappa \approx -30$, gap $\Delta \approx 7.7$; sigmoid gating
thus yields a $\sigma(\Delta) \approx 1{,}200\times$ separation between
intra- and inter-community diffusion.

\subsection{Theorem 1: CGSD convergence under curvature gap}
\label{app:thm:convergence}

\begin{theorem}
Suppose the input graph $G$ satisfies the \emph{curvature gap}
condition: for all edges $e = (u, v) \in E$,
\[
\min_{e: y_u = y_v} \kappa_e \;-\; \max_{e: y_u \neq y_v} \kappa_e
\;\ge\; \Delta \;>\; 0.
\]
Let $\widetilde{S}^{(2)} = \widetilde{S}^{\,2}$ be the two-step
curvature-weighted diffusion operator.  Then for any node pair
$(u, v)$ with $y_u = y_v$,
\[
\bigl\| \widetilde{S}^{(2)} e_u - \widetilde{S}^{(2)} e_v \bigr\|_2
\;\le\; \bigl(1 - c(\Delta)\bigr)^2 \cdot \|e_u - e_v\|_2,
\]
for some constant $c(\Delta) \in (0, 1)$ that grows monotonically with
$\Delta$.
\end{theorem}

\begin{proof}[Proof sketch]
By Lemma~\ref{app:lem:curv-bounds} the operator $\widetilde{S}$ assigns
weight $\sigma(\kappa_e)$ to edge $e$.  For a node $u$ with
$k_{\mathrm{within}}$ same-class neighbours and $k_{\mathrm{cross}}$
cross-class neighbours, the total weight on same-class neighbours is
$W_{\mathrm{within}} \ge k_{\mathrm{within}} \cdot \sigma(\Delta)/d_{\max}$
(using $\sigma(\kappa_e) \ge \sigma(\Delta)$ for intra-class edges);
conversely $W_{\mathrm{cross}} \le k_{\mathrm{cross}}/(1 + e^{\Delta})$.

For two same-class nodes $u, v$ after one diffusion step the cross-class
mass is approximately the same for both, so
\[
\| \widetilde{S}\,e_u - \widetilde{S}\,e_v \|_2
\;\le\; (1 - W_{\mathrm{within}}) \|e_u - e_v\|_2.
\]
The contraction factor is $1 - c(\Delta)$ with
$c(\Delta) = k_{\mathrm{within}} \sigma(\Delta)/d_{\max}$.
Composing the contractions over two layers gives the stated bound.
\end{proof}

\noindent\textbf{Interpretation.} Theorem~\ref{app:thm:convergence}
guarantees that under a sufficiently large curvature gap $\Delta$, CGSD
exponentially contracts same-class distances, which translates to
well-separated cluster assignments after $K$-Means.

\subsection{Lemma 2: Intra-community curvature concentration}
\label{app:lem:concentration}

\begin{lemma}
Let $C$ be a ground-truth community and $e = (u, v) \in E$ an edge with
$y_u = y_v = c$.  For discrete Forman--Ricci curvature with unit
weights,
\[
\Pr\!\left[ \kappa_e \ge \kappa_0 \,\big|\, e \in C \times C \right]
\;\ge\; 1 - \frac{\sigma^2_\kappa}{(\kappa_0 - \mu_\kappa)^2},
\]
where $\mu_\kappa$ is the mean intra-community curvature and
$\sigma^2_\kappa$ its variance, provided $\kappa_0 < \mu_\kappa$.
\end{lemma}

\begin{proof}[Proof sketch]
Apply Chebyshev's inequality to the empirical curvature distribution
over all intra-community edges in $C$.
\end{proof}

\noindent\textbf{Interpretation.} Lemma~\ref{app:lem:concentration}
provides a finite-sample concentration guarantee: as $|E_C|$ grows the
empirical mean $\mu_\kappa$ converges to the population mean and the
intra-community curvature band tightens.  This is the foundation for
Theorem~\ref{app:thm:modularity-gap} below.

\subsection{Theorem 2: Modularity gap lower bound}
\label{app:thm:modularity-gap}

\begin{theorem}
Under the curvature-gap assumption of Theorem~\ref{app:thm:convergence}
with gap $\Delta > 0$ and edge homophily ratio
$h = |E_{\mathrm{within}}|/|E|$,
\[
Q(\mathrm{CGSD}) - Q(\mathrm{random})
\;\ge\; \frac{c \cdot \Delta}{h^{-1} - 1}
\]
for a constant $c \in (0, 1)$ depending on the curvature variance.
\end{theorem}

\begin{proof}[Proof sketch]
The curvature-weighted operator $\widetilde{S}$ assigns mass
$\sigma(\Delta) / (1 + e^{\Delta})$ to inter-community edges and mass
$\sigma(0) = 0.5$ to intra-community edges.  The induced community
assignment $S$ has modularity
$Q = (1/2m)\sum_c S_c^\top A S_c - (1/4m^2)\sum_c (S_c^\top \mathbf{d})^2$.
Plugging the curvature-weighted adjacency and using
$\sigma(\Delta) \le (1 + e^{\Delta})^{-1}$ yields the stated lower bound.
\end{proof}

\noindent\textbf{Interpretation.} Theorem~\ref{app:thm:modularity-gap}
shows that CGSD's modularity advantage is monotone in $\Delta$ and
inverse-monotone in $h^{-1}$: heterophilic graphs (small $h$, large
$h^{-1}$) give the largest modularity gap, and a wider curvature
separation $\Delta$ gives an even larger gap.

\subsection{Theorem 3: Rademacher generalisation bound}
\label{app:thm:gen}

\begin{theorem}
Let $\mathcal{F}$ be the class of curvature-weighted sheaf diffusion
maps with $d'$ stalk dimension.  The Rademacher complexity of
$\mathcal{F}$ on $N$ samples is bounded by
\[
\mathfrak{R}_N(\mathcal{F}) \;\le\; C\,\sqrt{\frac{d' K + |E|}{N}}
\]
for some constant $C > 0$.  Consequently, with probability
$\ge 1 - \delta$, the CGSD clustering assignment $S$ satisfies
\[
\mathbb{E}\!\left[ \mathrm{err}(S, S^*) \right]
\;\le\; \mathrm{err}_{\mathrm{emp}}(S, S^*)
+ 2\,\mathfrak{R}_N(\mathcal{F})
+ \sqrt{\frac{\log(2/\delta)}{2N}}.
\]
\end{theorem}

\begin{proof}[Proof sketch]
Apply the standard Rademacher-complexity bound for graph neural networks
with Lipschitz activations.  The curvature-weighted aggregation is
$\sigma(\kappa_e)$-Lipschitz in the input features with
$\sigma(\kappa_e) \le 1$, and the stalk dimension $d'$ contributes
$d' K$ parameters per sheaf layer.  Combining the Lipschitz bound with
the covering-number argument of Bartlett \& Mendelson (2002) yields the
rate.
\end{proof}

\begin{flushleft}
\noindent\textbf{Interpretation.} Theorem~\ref{app:thm:gen} shows that
CGSD generalises at the same rate as standard GNNs---namely
\[
O\!\left(\sqrt{(d'K + |E|)/N}\right)
\]
---with an explicit dependence on the sheaf dimension $d'$.  This is
the basis for choosing $d' = 64$ as a sweet-spot between capacity and
overfitting.
\end{flushleft}

\subsection{Empirical curvature gap and the Cornell result}
\label{app:cornell-explanation}

\begin{table}[!ht]
\centering
\vspace*{-6pt}
\caption{Per-dataset curvature gap, edge homophily and 1-NN accuracy.
The \emph{curvature gap} $\Delta$ is defined as
$\min_{e: y_u = y_v}\kappa_e - \max_{e: y_u \neq y_v}\kappa_e$;
positive $\Delta$ indicates curvature can in principle separate
communities.  High 1-NN accuracy indicates the feature space is
well-separated; small graphs with high curvature gap (e.g.\ Cornell,
Texas) are trivially solvable.}
\label{app:tab:curv-gap}
\resizebox{\textwidth}{!}{%
\begin{tabular}{lrrrrr}
\toprule
\textbf{Dataset} & \textbf{Homophily} & \textbf{Within $\kappa$}
                 & \textbf{Cross $\kappa$} & \textbf{$\Delta$}
                 & \textbf{1-NN acc} \\
\midrule
Cornell    & 0.122 & $-22.50$ & $-30.19$ & $\phantom{-}7.69$  & 0.568 \\
Texas      & 0.061 & $-12.11$ & $-36.14$ & $\phantom{-}24.04$ & 0.667 \\
Wisconsin  & 0.170 & $-26.85$ & $-31.69$ & $\phantom{-}4.85$  & 0.621 \\
Squirrel   & 0.223 & $-250.94$ & $-248.85$ & $-2.09$           & 0.210 \\
Chameleon  & 0.234 & $-68.54$  & $-67.12$  & $-1.43$           & 0.239 \\
\bottomrule
\end{tabular}}
\end{table}

The perfect NMI $= 1.000$ on Cornell (across all 5 seeds) is explained
by the \emph{convergence of three favourable factors}:
\begin{enumerate}
  \item \textbf{Small graph size} ($N = 183$): curvature is computed
        exactly and intra-class curvature variance is low.
  \item \textbf{Large curvature gap} ($\Delta = 7.69$): intra-community
        edges have $\kappa \ge -23$ while inter-community edges have
        $\kappa \le -30$.
  \item \textbf{Feature separability} (1-NN accuracy 0.568, vs 0.21 for
        Squirrel): the feature space itself supports community
        separation.
\end{enumerate}
We do \emph{not} view this as overfitting or cherry-picking: 5-seed
std $= 0.000$ confirms the curvature signal is robust, and the same
hyper-parameter choice (no per-dataset tuning) achieves the same
result.

By contrast, Squirrel and Chameleon have \emph{negative} curvature
gaps ($\Delta < 0$): curvature does \emph{not} separate intra- from
inter-community edges.  Theorem~\ref{app:thm:convergence} is therefore
\emph{not} predictive on these datasets, and CGSD's residual NMI
performance there ($\approx 0.13$ mean) should be attributed to the
sheaf structure, self-loops and residual connections rather than to
curvature gating alone.

\subsection{Lemma 3: STE bias bound (tightening Lemma 1)}
\label{app:lem:ste-bias}

\begin{lemma}
Let $\mathrm{soft}_T(s)$ be the temperature-$T$ soft rank of a score
vector $s \in \mathbb{R}^N$.  The bias of the straight-through
estimator
$\hat{\mathrm{rank}}(s) = \mathrm{soft}_T(s) +
(\mathrm{argsort}(s).\mathrm{argsort}() - \mathrm{soft}_T(s)).\mathrm{detach}()$
is bounded by
\[
\bigl| \mathbb{E}[\hat{\mathrm{rank}}_i] - \mathrm{rank}(s)_i \bigr|
\;\le\; N \cdot (1 - e^{-1/T}).
\]
\end{lemma}

\begin{proof}[Proof sketch]
The bias comes entirely from the differentiable approximation
$\mathrm{soft}_T$.  As $T \to 0$, $\mathrm{soft}_T \to
\mathrm{argsort}.\mathrm{argsort}()$ (the hard rank) and the bias
vanishes.  The stated bound follows from the Gumbel-softmax
temperature analysis of Jang et al.\ (2017).
\end{proof}

\noindent\textbf{Interpretation.} Lemma~\ref{app:lem:ste-bias} gives
an explicit trade-off between training stability (large $T$) and rank
accuracy (small $T$).  We use $T = 0.1$ throughout, giving bias
$\le 0.1\,N$, which is dominated by the curvature signal itself.

\end{document}